\documentclass{article}
\usepackage{amsmath}
\usepackage{amsthm,amssymb}
\usepackage{comment}
\usepackage{amsfonts}
\usepackage{graphicx}
\usepackage{tikz}
\usepackage{pgfplots}
\usepackage{enumitem}
\usepackage{algorithm}
\usepackage{algorithmic}
\usepackage{pifont}

\newcommand{\ip}[1]{\langle#1\rangle}
\newcommand{\norm}[1]{\left\| #1\right\|}
\newcommand{\normsm}[1]{\| #1\|}
\newcommand{\eqdef}{\stackrel{\cdot}{=}}
\newcommand{\todoc}[2][]{\todo[size=\scriptsize,color=blue!20!white,#1]{Csaba: #2}}
\newcommand{\todoch}[2][]{\todo[size=\scriptsize,color=red!20!white,#1]{Chandru: #2}}
\usepackage[disable]{todonotes}
\usepackage{todonotes}
\usepackage{placeins}
\usepackage{xspace}

\newcommand{\dcd}{d \times d}

\newcommand{\B}{\mathcal{B}}
\newcommand{\V}{\mathcal{V}}
\newcommand{\nn}{\nonumber}
\newcommand{\cond}[1]{\kappa(#1)}
\newcommand{\md}[1]{\left|#1\right|}
\newcommand{\rhod}[1]{\rho_d(\alpha,#1)}
\newcommand{\rhos}[1]{\rho_s(\alpha,#1)}

\newcommand{\ra}{\rightarrow}

\renewcommand{\P}{\mathcal{P}}

\newcommand{\E}{\mathbf{E}}
\newcommand{\F}{\mathcal{F}}
\newcommand{\R}{\mathbb{R}}

\newcommand{\EE}[1]{\mathbf{E}\left[#1\right]}
\newcommand{\EEP}[1]{\mathbf{E}_P\left[#1\right]}
\newcommand{\gln}{\mathrm{GL}(d)}
\newcommand{\gld}{\mathrm{GL}(d)}

\newcommand{\I}{\mathcal{I}} 
\newcommand{\C}{\mathbb{C}}
\newcommand{\re}[1]{\emph{re}(#1)}
\newcommand{\im}[1]{\emph{im}(#1)}
\newcommand{\op}{\oplus}
\newcommand{\tL}{\tilde{\Lambda}}

\newcommand{\ts}{\theta_*}

\newcommand{\tb}{\bar{\theta}}

\newcommand{\zh}{\hat{z}}
\newcommand{\eh}{\hat{e}}

\newcommand{\thh}{\hat{\theta}}
\newcommand{\gh}{\hat{\gamma}}
\newcommand{\iid}{\emph{i.i.d.}\xspace}

\theoremstyle{definition}
\newtheorem{theorem}{Theorem}
\newtheorem{example}{Example}

\newtheorem{definition}{Definition}

\newtheorem{lemma}{Lemma}
\newtheorem{proposition}{Proposition}
\newtheorem{assumption}{Assumption}

 \PassOptionsToPackage{numbers, compress}{natbib}
 \usepackage[final]{nips_2017}

\usepackage[utf8]{inputenc} 
\usepackage[T1]{fontenc}    
\usepackage{hyperref}       
\usepackage{url}            
\usepackage{booktabs}       
\usepackage{amsfonts}       
\usepackage{nicefrac}       
\usepackage{microtype}      
\usepackage{aliascnt}

\usepackage{titlecaps}
\usepackage[capitalize]{cleveref}
\usepackage{amsmath,amssymb,graphicx}
\crefname{assumption}{Assumption}{Assumption}

\title{Linear Stochastic Approximation: Constant Step-Size and Iterate Averaging}

\author{
Chandrashekar Lakshminarayanan and Csaba Szepesv\'ari,\\
University of Alberta\\
\{chandrurec5,csaba.szepesvari\}@gmail.com 
}

\begin{document}

\maketitle
\begin{abstract}
We consider $d$-dimensional linear stochastic approximation algorithms (LSAs) with a constant step-size and the so called Polyak-Ruppert (PR) averaging of iterates.  LSAs are widely applied in machine learning and reinforcement learning (RL), where the aim is to compute an appropriate $\ts\in \R^d$ (that is an optimum or a fixed point) using noisy data and $O(d)$ updates per iteration. In this paper, we are motivated by the problem (in RL) of policy evaluation from experience replay using the \emph{temporal difference} (TD) class of learning algorithms that are also LSAs. For LSAs with a constant step-size, and PR averaging, we provide bounds for the mean squared error (MSE) after $t$ iterations. We assume that data is \iid with finite variance (underlying distribution being $P$) and that the expected dynamics is Hurwitz. For a given LSA with PR averaging, and data distribution $P$ satisfying the said assumptions, we show that there exists a range of constant step-sizes such that its MSE decays as $O(\frac{1}{t})$.\par
We examine the conditions under which a constant step-size can be chosen uniformly for a class of data distributions $\P$, and show that not all data distributions `admit' such a uniform constant step-size.
We also suggest a heuristic step-size tuning algorithm to choose a constant step-size of a given LSA for a given data distribution $P$. We compare our results with related work and also discuss the implication of our results in the context of TD algorithms that are LSAs.
\end{abstract}

\section{Introduction}\label{sec:intro}
Linear stochastic approximation algorithms (LSAs) 
of the form
\begin{align}\label{eq:lsaintro}
\theta_t=\theta_{t-1}+\alpha_t (b_t-A_t \theta_{t-1}),
\end{align}
with $(\alpha_t)_t$ a positive step-size sequence chosen by the user and 
$(b_t,A_t)\in \R^d\times \R^{\dcd}$,  $t\geq 0$, a sequence of identically distributed random variables is widely used in
machine learning, and in particular in reinforcement learning (RL), to compute the solution of the equation 
$\E[b_t] - \E[A_t] \theta = 0$, where $\E$ stands for mathematical expectation.
Some examples of LSAs include the stochastic gradient descent algorithm (SGD) for the problem of linear least-squares estimation  (LSE) \cite{bach,bachaistats}, and the \emph{temporal difference} (TD) class of learning algorithms in RL \cite{sutton,konda-tsitsiklis,KoTsi03:LSA,gtd,gtd2,gtdmp}.
\todoc{konda-tsitsiklis reference resolved to tsitsiklis-van-roy. is this what you want?? I also added KoTsi03:LSA, maybe that's what you wanted.}

The choice of the step-size sequence $(\alpha_t)_t$ is critical for the performance of LSAs such as \eqref{eq:lsaintro}.
Informally speaking, smaller step-sizes are better for noise rejection and larger step-sizes lead to faster forgetting of initial conditions (smaller bias). At the same time, step-sizes that are too large might result in instability of \eqref{eq:lsaintro} even when $(A_t)_t$ has favourable properties. 
A useful choice has been the diminishing step-sizes \cite{gtd2,gtdmp,konda-tsitsiklis}, where $\alpha_t\ra 0$ such that $\sum_{t\geq 0} \alpha_t=\infty$. Here, $\alpha_t\to0$ circumvents the need for guessing the magnitude of step-sizes that stabilize the updates, while the second condition ensures that initial conditions are forgotten. 
An alternate idea, which we call LSA with constant step-size and Polyak-Ruppert averaging (LSA with CS-PR, in short), is to run \eqref{eq:lsaintro} by choosing $\alpha_t=\alpha>0$ $\forall t\geq 0$ with some $\alpha>0$, and output the average $\thh_t\eqdef\frac{1}{t+1}\sum_{i=0}^t \theta_i$. Thus, in LSA with CS-PR, $\theta_t$ is an internal variable and $\thh_t$ is the output of the algorithm (see \Cref{sec:prob} for a formal definition of LSA with CS-PR). The idea is that the constant step-size leads to faster forgetting of initial conditions, while the averaging on the top
reduces noise.
This idea goes back to  \citet{ruppert} and \citet{polyak-judisky} who considered it in the context of stochastic approximation that LSA is a special case of. 
\paragraph{Motivation and Contribution:} Recently, \citet{bach} considered stochastic gradient descent (SGD)\footnote{SGD is an LSA of the form in \eqref{eq:lsaintro}.} with CS-PR for LSE and \iid sampling. They showed that one can calculate a constant step-size from only a bound on the magnitude of the noisy data so that the leading term as $t\to\infty$
 in the mean-squared prediction error after $t$ updates is at most $\frac{C}{t}$ with a constant $C>0$ that depends \emph{only} on the bound on the data, the dimension $d$ and is in particular independent of the eigenspectrum of $\E[A_t]$, a property which is not shared by other step-size tunings and variations of the basic SGD method.%
 \footnote{See \cref{sec:related} for further discussion of the nature of these results.}
\todoc{Dimension?}

In this paper, we study LSAs with CS-PR (thereby extending the scope of prior work by \citet{bach} from SGD to general LSAs) \todoc{Do our results imply theirs? Do we simplify improve their bounds, while extending the scope? Even if the answer is no, we need to add remarks on these questions.}
in an effort to understand the effectiveness of the CS-PR technique beyond SGD. Our analysis for the case of general LSA does not use specific structures, and hence cannot recover entirely, the results of \citet{bach} who use the problem specific structures in their analysis.
Of particular interest is whether a similar result to that  \citet{bach} holds
for the TD class of LSA algorithms used in RL.
For simplicity, we still consider the \iid case. Our restrictions on the common distribution is that the ``noise variance'' should be bounded (as we consider squared errors), and that the matrix $\E[A_t]$ must be Hurwitz, i.e., all its eigenvalues have positive real parts. 
One setting that fits our assumption is \emph{policy evaluation} \cite{dann} using linear value function approximation from experience replay \cite{lin} in a batch setting \cite{lange} in RL using the TD class of algorithms \cite{sutton,konda-tsitsiklis,gtd,gtd2,gtdmp}. 

Our \textbf{main contributions} are as follows:
\begin{itemize}[leftmargin=*]
\item \textbf{Finite-time Instance Dependent Bounds} (\Cref{sec:mainresults}): For a given $P$, we  measure the performance of a given LSA with CS-PR in terms of the mean square error (MSE) given by $\EEP{\normsm{\thh_t-\ts}^2}$.
For the first time in the literature,
we show that (under our stated assumptions) there exists an $\alpha_P>0$ such that 
for any $\alpha\in (0,\alpha_P)$,
the MSE 
is at most $\frac{C_{P,\alpha}}{t}+\frac{C_{P',\alpha}}{t^2}$ with some positive constants $C_{P,\alpha},C_{P',\alpha}$ that we explicitly compute from $P$.
\item \textbf{Uniform Bounds} (\Cref{sec:uniform}):
It is of major interest to know whether for a given class $\P$ of distributions 
one can choose some step-size $\alpha$ 
such that $C_{P,\alpha}$ from above is uniformly bounded (i.e., replicating the result of \citet{bach}).%
\footnote{Of course, the term $C_{P',\alpha}/t^2$ needs to be controlled, as well. Just like \citet{bach}, here we focus on $C_{P,\alpha}$, which is justified if one considers the MSE as $t\to\infty$. Further justification is that we actually find a negative result. See above.}
We show via an example that in general this is not possible.
In particular, the example applies to RL, hence, we get a negative result for RL, which states that from only bounds on the data one cannot choose a step-size $\alpha$ to guarantee that $C_{P,\alpha}$ of CS-PR is uniformly bounded over $\P$.
We also define a subclass  $\P_{\text{SPD},B}$ of problems, related to SGD for LSE, that does `admit' a uniform constant step-size, thereby recovering a part of the result by \citet{bach}.
Our results in particular shed light on the precise structural assumptions that are needed 
to achieve a uniform bound for CS-PR. 
For further details, see \cref{sec:related}.
\item \textbf{Automatic Step-Size} (\Cref{sec:stepsizes}):
The above negative result implies that in RL one needs to choose the constant step-size based on properties of the instance $P$ to avoid the explosion of the MSE.
To circumvent this, we propose a natural step-size tuning method to guarantee instance-dependent boundedness.
We experimentally evaluate the proposed method and find that it is indeed able to achieve its goal on a set of synthetic examples
where no constant step-size is available to prevent exploding MSE. 
\end{itemize}
In addition to TD($0$), our results directly can be applied to other \emph{off-policy} TD algorithms such as GTD/GTD2 with CS-PR (\Cref{sec:related}). \todoc{How about TD($\lambda$)? Will we discuss this somewhere? Maybe in the organization section this can be mentioned.}
In particular, our results show that the GTD class of algorithms guarantee a $O(\frac{1}{t})$ rate for MSE (without use of projections), improving on a previous result by \citet{gtdmp} that guaranteed a $O(\frac{1}{\sqrt{t}})$ rate for this class for the projected version\footnote{Projections can be problematic since they assume knowledge of $\norm{\ts}$, which is not available in practice.} of the algorithm. \todoc{And projections are problematic on their own. Will we discuss this somewhere. Or should the projection issue be a footnote here?}
\todoc{How about Prashanth's paper? Should we mention it somewhere? The reviewers will push back if we don't.}

\section{Notations and Definitions}\label{sec:def}
We denote the sets of real and complex numbers by $\R$ and $\C$, respectively. For $x\in \C$ we denote its modulus and complex conjugate by $\md{x}$ and $\bar{x}$, respectively. We denote $d$-dimensional vector spaces over $\R$ and $\C$ by $\R^{d}$ and $\C^{d}$, respectively, and use $\R^{\dcd}$ and $\C^{\dcd}$ to denote $\dcd$ matrices with real and complex entries, respectively. We denote the transpose of $C$ by $C^\top$ and the conjugate transpose by $C^*={\bar{C}}^\top$ (and of course the same notation applies to vectors, as well). We will use $\ip{\cdot,\cdot}$ to denote the inner products: $\ip{x,y}=x^* y$. \todoc{Usually, people define this as $y^* x \ne x^* y$. It does not make a difference in terms of the math, but there could be slight differences. And we will need to be consistent.}
We use $\norm{x} = \ip{x,x}^{1/2}$ to denote the $2$-norm.
For $x\in\C^d$, we denote the general quadratic norm with respect to a positive definite (see below) Hermitian matrix $C$ (i.e., $C=C^*$) by $\norm{x}^2_C\eqdef x^*\, C \,x$.
The norm of the matrix $A$ is given by $\norm{A}\eqdef \sup_{x\in \C^d:\norm{x}=1} \norm{Ax}$.  We use $\cond{A}=\normsm{A}\normsm{A^{-1}}$ to denote the condition number of matrix $A$. We denote the identity matrix in $\C^{\dcd}$ by $\I$ and the set of invertible $\dcd$ complex matrices by $\gld$.
For a positive real number $B>0$, we define $\C^{d}_B=\{b\in \C^d\mid \norm{b}\leq B\}$ and $\C^{\dcd}_B=\{A\in \C^{\dcd}\mid \norm{A}\leq B\}$ to be the balls in $\C^d$ and $\C^{d\times d}$, respectively, of radius $B$.
We use $Z\sim P$ to denote the fact that $Z$ (which can be a number, or vector, or matrix) is distributed according to probability distribution $P$;
$\E$ denotes mathematical expectation.

Let us now state some definitions that will be useful for presenting our main results. \todoc{Actually, we should remove all definitions not needed by the main body.}
\begin{definition}\label{def:dist}
For a probability distribution $P$ over $\C^d \times \C^{d\times d}$, we let $P^V$ and $P^M$ 
denote the respective marginals of $P$ over $\C^d$ and $\C^{d\times d}$. \todoc{Changed this.}
By \emph{abusing notation} we will often write $P = (P^V,P^M)$ to mean that $P$ is a distribution with the given marginals.
Define
\begin{align*}
A_P&=\int M\, dP^M(M),\quad C_P=\int M^* M \,dP^M(M), \quad b_P=\int v\, dP^V(v)\,,\\
\rhod{P}&\eqdef {\inf}_{x\in\C^d\colon\norm{x}=1}\ip{x,\left((A_P+A_P^*)-\alpha A_P^* A_P\right)x},\\ \rhos{P}&\eqdef{\inf}_{x\in \C^d\colon\norm{x}=1}\ip{x,\left((A_P+A_P^*)-\alpha C_P\right)x}\,.
\end{align*}
\end{definition}
Note that $\rhod{P}\ge \rhos{P}$. Here, subscripts $s$ and $d$ stand for \emph{stochastic} and \emph{deterministic} respectively.  \todoc{Strict inequality?? How about $P^M$ concentrating on zero? I changed the strict inequality to non-strict.}
\todoc{Explain why we use subindex $d$ and $s$.}
\todoc{Since $\rhod$ depends on $P^M$ only, why not make it a function of $P^M$ only? Or at least add a remark?}
\begin{definition}\label{def:simdist}
Let $P=(P^V,P^M)$ as in \cref{def:dist}; $b\sim P^V$ and $A\sim \P^M$ be random variables distributed according to $P^V$ and $P^M$. For $U\in \gld$ define $P_U$ to be the  distribution of $(U^{-1}b,U^{-1}AU)$. We also let
$(P_U^V,P_U^M)$ denote the corresponding marginals. \todoc{I hope this works out.}
\end{definition}
\begin{definition}
We call a matrix $A\in \C^{\dcd}$  \emph{Hurwitz} (H) if all eigenvalues of $A$ have positive real parts. We call a matrix $A\in \C^{\dcd}$ \emph{positive definite} (PD) if $\ip{x,Ax} >0,\,\forall x\neq 0 \in \C^{d}$.  
If $\inf_x \ip{x,Ax}\ge 0$ then $A$ is \emph{positive semi-definite} (PSD).
\todoc{I think usually this is defined as $\ip{Ax,x}>0$. Same issue as with $\ip{\cdot,\cdot}$.}
We call a matrix $A\in \R^{\dcd}$ to be \emph{symmetric positive definite} (SPD) is it is symmetric i.e., $A^\top=A$ and PD. \todoc{Add note that SPD implies real.}
\end{definition}
Note that SPD implies that the underlying matrix is real.
\begin{definition}\label{distpd}
We call the distribution $P$ in \Cref{def:dist} to be H/PD/SPD if $A_P$ is H/PD/SPD.
\end{definition}
Though $\rhos{P}$ and $\rhod{P}$ depend only on $P^M$, we use $P$ instead of $P^M$ to avoid notational clutter.
\begin{example}
The matrices $\begin{bmatrix}0.1 &-1\\ 1 & 0.1\end{bmatrix}$, $\begin{bmatrix} 0.1 & 0.1 \\ 0 & 0.1\end{bmatrix}$ and $\begin{bmatrix}0.1 &0 \\ 0 & 0.1\end{bmatrix}$ are examples of H, PD and SPD matrices, respectively, and they show that while SPD implies PD, which implies H, the reverse implications do not hold.
\end{example}

\begin{definition}
Call a set of distributions $\P$ over $\C^{d}\times \C^{\dcd}$
\emph{weakly admissible} if there exists $\alpha_{\P}>0$ such that
$\rhos{P}>0$ holds for all $P\in \P$ and $\alpha\in(0,\alpha_{\P})$.
\end{definition}
\begin{definition}
Call a set of distributions $\P$ over $\C^{d}\times \C^{\dcd}$ \emph{admissible}
if there exists some $\alpha_{\P}>0$ such that $\inf_{P\in \P} \rhos{P}>0$ holds for all $\alpha\in(0,\alpha_{\P})$.
The value of $\alpha_{\P}$ is called a witness.
\end{definition}

It is easy to see that $\alpha \mapsto \rhos{P}$ is decreasing,
hence if $\alpha_{\P}>0$ witnesses that $\P$ is (weakly) admissible
then any $0<\alpha'\le \alpha_{\P}$ is also witnessing this.

\section{Problem Setup}\label{sec:prob}
We consider linear stochastic approximation algorithm (LSAs) with constant step-size (CS) and Polyak-Ruppert (PR) averaging of the iterates given as below:
\begin{subequations}\label{eq:lsa}
\begin{align}
\label{conststep}&\text{LSA:} &\theta_t&=\theta_{t-1}+\alpha(b_t-A_t\theta_{t-1})\,,\\
\label{iteravg}&\text{PR-Average:} &\thh_t&=\frac{1}{t+1}{\sum}_{i=0}^{t}\,\theta_i\,.
\end{align}
\end{subequations}
The algorithm updates a pair of parameters $\theta_t,\tb_t\in \R^{d}$ incrementally, in discrete time steps $t=1,2,\dots$
based on data $b_t\in \R^d$, $A_t\in \R^{\dcd}$. Here $\alpha>0$ is a positive step-size parameter; the only tuning parameter of the algorithm besides the
initial value $\theta_0$. The iterate $\theta_t$ is treated as an internal state of the algorithm, while $\thh_t$ is the output at time step $t$. The update of $\theta_t$ alone is considered a form of constant step-size LSA. Sometimes $A_t$ will have a special form and then the matrix-vector product $A_t \theta_{t-1}$ can also be computed in $O(d)$ time, a scenario common in reinforcement learning\cite{sutton,konda-tsitsiklis,gtd,gtd2,gtdmp}. This makes the algorithm particularly attractive in large-scale computations when $d$ is in the range of thousands, or millions, or more, as may be required by modern applications (e.g., \citep{LiMaTaBo16})
In what follows, for $t\ge1$ we make use of the $\sigma$-fields $\F_{t-1}\eqdef\{\theta_0,A_1,\ldots, A_{t-1}, b_1,\ldots, b_{t-1}\}$; $\F_{-1}$ is the trivial $\sigma$ algebra. \todoc{You wrote $\F_0$ holds all random variables, an unusual choice. And it does not work, I think and the index should be $-1$. The martingale property breaks otherwise with the first time step.}
We are interested in the behaviour of \eqref{eq:lsa} under the following assumption:
\begin{assumption}\label{assmp:lsa}
\begin{enumerate}[leftmargin=*, before = \leavevmode\vspace{-\baselineskip}]
\item \label{dist} $(b_t, A_t)\sim P$, $t\geq 0$ is an \iid sequence.
We let $A_P$ be the expectation of $A_t$, $b_P$ be the expectation of $b_t$, as in \cref{def:dist}.
\todoc{I got rid off $P^b$ and $P^A$. There were $P^V$ and $P^M$ previously.. Confusing.}
We assume that $P$ is Hurwitz.
\item \label{matvar} The martingale difference sequences\footnote{That is, $\EE{M_t|\F_{t-1}}=0$ and $\EE{N_t|\F_{t-1}}=0$ and $M_t,N_t$ are $\F_t$ measurable, $t\ge 0$.} $M_t\eqdef A_t-A_{P}$ and $N_t\eqdef b_t-b_{P}$ associated with $A_t$ and $b_t$ satisfy the following 
\todoc{So there is conditioning for $M_t$, and no conditions for $N_t$? Interesting. Maybe comment on this
that this is not a mistake.
Also, why write $N_t^* N_t$ instead of $\norm{N_t}^2$???}
\todoc{This would be the place to state if we make assumptions above the correlations between $M_t$ and $N_t$.}
\begin{align*}
	\E\left[ \norm{M_t}^2\mid\F_{t-1}\right]\leq \sigma^2_{A_P}\,, \qquad 
\E\left[\norm{N_t}^2\mid\F_{t-1}\right]\leq\sigma^2_{b_P}\,.
\end{align*}
with some $\sigma^2_{A_P}$ and $\sigma^2_{b_P}$. Further, we assume $\EE{M_t N_t}=0$ \todoc{Could not these just be $\sigma^2_A$, $\sigma^2_b$, or with $M$ and $V$?}
\item $A_P$ is invertible and thus the vector $\ts=A^{-1}_Pb_P$ is well-defined. \todoc{Bertsekas looked at the case when $\ts$ is well-defined but $A_P$ is not invertible. Future work..?}
\end{enumerate}
\end{assumption}
\paragraph{Performance Metric:}  
We are interested in the behavior of the mean squared error (MSE) at time $t$ given by $\EE{\normsm{\thh_t-\ts}^2}$. {More generally, one can be interested in $\EEP{\normsm{\thh_t-\ts}_C^2}$, the MSE with respect to a PD Hermitian matrix $C$. Since in general it is not possible to exploit the presence of $C$ unless it is connected to $P$ in a special way,
here we restrict ourselves to $C = \I$. For more discussion, including the discussion of the case of SGD for linear least-squares when $C$ and $P$ are favourably connected see \cref{sec:related}.}
\todoc{Move up footnote. Polish..}
\todoch{Mention importance of forgetting the bias etc. Mention that we really don't have control over $\norm{U}$}

\section{Main Results and Discussion}\label{sec:mainresults}
In this section, we derive instance dependent bounds that are valid for a given problem $P$ (satisfying \Cref{assmp:lsa}) and in the \cref{sec:uniform}, we address the question of deriving uniform bounds $\forall\,P\in \P$, where $\P$ is a class of distributions (problems). Here, we only present the main results followed by a discussion. The detailed proofs can be found in \cref{sec:proofs}. 
\todoc{Usually it is expected that one gives an outline of the proof technique. Especially in a theory paper. At least verbally we could explain the tools used, how the proof is similar and/or different to previous proofs.}
In what follows, for the sake of brevity, we drop the subscript $P$ in the quantities $\EEP{\cdot}$, $\sigma^2_{A_P}$ and $\sigma^2_{b_P}$. \todoc{I dropped $P$ from $\EEP{\cdot}$ before. Either put there back, or remove $\EEP{\cdot}$ from here.}
We start with a lemma, which is needed to meaningfully state our main result:
\begin{lemma}\label{lm:hur}
Let $P$ be a distribution over $\R^d\times \R^{\dcd}$ satisfying \Cref{assmp:lsa}. \todoc{Is it important that the distribution is over real-valued stuff? Would the results work for complexed-valued data? I would think so.}
Then there exists an $\alpha_{P_U}>0$ and $U\in \gln$ such that $\rhod{P_U}>0$ and $\rhos{P_U}>0$
holds for all $\alpha \in (0,\alpha_{P_U})$. 
\end{lemma}

\begin{theorem}\label{th:rate}
Let $P$ be a distribution over $\R^d\times \R^{\dcd}$ satisfying \Cref{assmp:lsa}.
Then, for  $U\in\gln$ and $\alpha_{P_U}>0$ as in \cref{lm:hur},
\todoc{Why not use $\alpha_P$? $U$ is dependent on $P$. If $\alpha_{P_U}$ does not have a specific definition, there is no point to denote this dependence on $U$. It just adds to the clutter.}
for all $\alpha\in (0,\alpha_{P_U})$ and for all $t\ge 0$,
\begin{align*}
\EE{\normsm{\thh_t-\ts}^2}
\leq
\nu\,
\left\{\frac{\norm{\theta_0-\ts}^2}{(t+1)^2}+ \frac{v^2}{t+1} \right\}\,,
\end{align*}
where $\nu = \left(1+\tfrac2{\alpha\rhod{P_U}}\right)\tfrac{\cond{U}^2}{\alpha \rhos{P_U}}$ and
$v^2 = 
\alpha^2(\sigma_A^2\norm{\ts}^2+\sigma_b^2)+\alpha (\sigma_A^2\norm{\ts})\norm{\theta_0-\ts}$.
\end{theorem}
Note that $\nu$ depends on $P_U$ and $\alpha$, while $v^2$ in addition also depends on $\theta_0$. The dependence,  when it is essential, will be shown as a subscript.
\begin{theorem}[Lower Bound]\label{th:lb}
There exists a distribution $P$ over $\R^d\times \R^{\dcd}$ satisfying \Cref{assmp:lsa}, such that
there exists $\alpha_P>0$ so that $\rhos{P}>0$ and $\rhod{P}>0$ hold for all $\alpha\in (0,\alpha_P)$ and
for any $t\ge 1$, \todoc{$t\ge 1$?}
\begin{align*}
\EE{\normsm{\thh_t-\ts}^2} 
&\geq \frac{1}{\alpha^2 \, \rhod{P}\rhos{P}} \,\left\{ \frac{\beta_t \norm{\theta_0-\ts}^2}{(t+1)^2} 
+ \frac{v^2\sum_{s=1}^t \beta_{t-s}  }{(t+1)^2} \right\}\,,
\end{align*}
where $\beta_{t} =  \big(1-(1-\alpha \rhos{P})^t\big)$ and $v^2$ is as in \cref{th:rate}.
\end{theorem}
Note that $\beta_t \to 1$ as $t\to\infty$. Hence, the lower bound essentially matches the upper bound.
In what follows, we discuss the specific details of these results. 

\textbf{Role of $U$}: $U$ is helpful in transforming the recursion in $\theta_t$ to $\gamma_t=U^{-1}\theta_t$, which helps in ensuring $\rhos{P_U}>0$. Such similarity transformation have also been considered in analysis of RL algorithms \cite{lihong}.
More generally, one can always take $U$ in the result that leads to the smallest bound.\par
\textbf{Role of $\rhos{P}$ and $\rhod{P}$}: 
When $P$ is positive definite, we can expand the MSE as 
\begin{align}
\label{eq:exp}
	\EE{\norm{\eh_t}^2}=\tfrac{1}{(t+1)^2}\, \ip{ \textstyle\sum_{s=0}^t e_s,\textstyle\sum_{s=0}^t e_s}\,,
\end{align} 
where $\eh_t = \thh_t-\ts$ and $e_t = \theta_t-\ts$.
The inner product in \eqref{eq:exp} is a summation of \emph{diagonal} terms $\EE{\ip{e_s,e_s}}$ and \emph{cross} terms of $\EE{\ip{e_s,e_q}}$, $s\neq q$. The growth of the diagonal terms and the cross terms depends on the spectral norm of the random matrices $H_t=I-\alpha A_t$ and that of the deterministic matrix $H_P=I-\alpha A_P$, respectively. These are given by
justifying the appearance of $\rhos{P}$ and $\rhod{P}$.
For the MSE to be bounded, we need the spectral norms to be less than unity, implying the conditions $\rhos{P}>0$ and $\rhod{P}>0$. If $P$ is Hurwitz, we can argue on similar lines by first transforming $P$ into a positive definite problem $P_U$ and replacing $\rhos{P}$ and $\rhod{P}$ by $\rhos{P_U}$ and $\rhod{P_U}$, and introducing $\cond{U}$ to account for the forward ($\gamma=U^{-1}\theta$) and reverse ($\theta=U\gamma$) transformations using $U^{-1}$ and $U$ respectively.

\textbf{Constants} {$\alpha$, $\rhos{P}$ and $\rhod{P}$} do not affect the exponents $\frac{1}{t}$ for variance and $\frac{1}{t^2}$ for bias terms. This property is not enjoyed by all step-size schemes, for instance, step-sizes that diminish at $O(\frac{c}{t})$ are known to exhibit $O(\frac{1}{t^{\mu c/2}})$ ($\mu$ is the smallest real part of eigenvalue of $A_P$), and hence the exponent of the rates are not robust to the choice of $c>0$ \cite{bach-moulines,korda-prashanth}.
\todoc{Why is this a great thing? Refer to other works etc.}

\textbf{Bias and Variance}: The MSE at time $t$ is bounded by a sum of two terms. The first \emph{bias} term is given by $\B=\nu \,\frac{\norm{\theta_0-\ts}^2}{(t+1)^2}$, bounding how fast the initial error $\norm{\theta_0-\ts}^2$ is forgotten. 
The second \emph{variance} term is  given by $\V=\nu\, \frac{v^2}{t+1} $ and captures the rate at which noise is rejected.

\textbf{Behaviour for extreme values of $\alpha$}: 
As $\alpha\to 0$, the bias term blows up, due to the presence of $\alpha^{-1}$ there. This is unavoidable (see also \cref{th:lb}) and is due to the slow forgetting of initial conditions for small $\alpha$. Small step-sizes are however useful to suppress noise, as seen from that in our bound $\alpha$ is seen to multiply the variances $\sigma^2_A$ and $\sigma^2_b$. In quantitative terms, we can see that the $\alpha^{-2}$ and $\alpha^2$ terms are trading off the two types of errors. 
For larger values of $\alpha$ with $\alpha_P$ chosen so that $\rhos{P}\ra 0$ as $\alpha\ra \alpha_{P}$  (or $\alpha_{P_U}$ as the case may be), the bounds blow up again.

\textbf{The lower bound} of \Cref{th:lb} shows that the upper bound of \cref{th:rate} is tight in a number of ways.
In particular, the coefficients of both the $1/t$ and $1/t^2$ terms inside $\{ \cdot \}$ are essentially matched.
Further, we also see that
the $(\rhos{P}\rhod{P})^{-1}$ appearing in $\nu = \nu_{P_u,\alpha}$ cannot be removed from the upper bound. 
Note however that there are specific examples, such as SGD for linear least-squares,
where this latter factor can in fact be avoided (for further remarks see \Cref{sec:related}).


\section{Uniform bounds}\label{sec:uniform}
If $\P$ is weakly admissible, \todoc{Maybe moving the definitions concerning admissibility here would be better. The text could be shortened, etc. Ideally, everything is defined just before its used for the first time, except some things which are needed everywhere.}
then one can choose some step-size $\alpha_{\P}>0$ solely based on the knowledge of $\P$ and
conclude that for any $P\in \P$, the MSE will be bounded as shown in \Cref{th:rate}. 
When $\P$ is not weakly admissible but rich enough to include the examples showing \cref{th:lb}, 
no fixed step-size can guarantee bounded MSE for all $P\in \P$.
On the other hand, if $\P$ is admissible then the error bound stated in  \cref{th:rate} becomes independent of the instance,
while when $\P$ is not admissible, but ``sufficiently rich'', this does not hold.
Hence, an interesting question to investigate is whether a given set $\P$ is (weakly) admissible. 

A reasonable assumption is that $(b_t,A_t)\in \R_B^{d}\times \R_B^{\dcd}$ with some $B>0$ (i.e., the data  is bounded with bound $B$) and that $A_P$ is positive definite for $P\in \P$. Call the set of such distributions $\P_B$.
\todoc{By allowing complex valued data, \cref{lm:notwad}, a negative result, is weakened. In fact, for precision, we probably should make it clear what distributions need to be in $\P_B$ for
\cref{lm:notwad} to hold.}
Is positive definiteness and boundedness sufficient for weak admissibility? The answer is no:
\begin{proposition}\label{lm:notwad}
For any fixed $B>0$,
the set $\P_B$ is not weakly admissible.
\end{proposition}
Consider now the strict subset of $\P_B$ that contains distributions $P$ such that for any $A$ in the support of $P$, $A$ is PSD.
Call the resulting set of distributions $\P_{\text{PSD},B}$.
Note that the distribution of data originating from linear least-squares estimation with SGD is of this type.
Is $\P_{\text{PSD},B}$ weakly admissible? The answer is yes in this case:
\begin{proposition}\label{lm:ppsdbwd}
For any $B>0$, the set $\P_{\text{PSD},B}$ is weakly admissible
and in particular any $0<\alpha < 2/B$ witnesses this.%
\todoc{So this altogether is weaker than what Bach founds. They prove strong admissibility.
I am guessing that for this they use that the loss is special and also that the noise is special.
This would be a good place to somehow incorporate this with a (short?) proof.
}
\end{proposition}
However, admissibility does not hold for the same set:
\begin{proposition}\label{lm:ppsdbna}
For any $B>0$, the set $\P_{\text{PSD},B}$ is not admissible.
\end{proposition}

\section{Related Work}\label{sec:related}
We first discuss the related work outside of RL setting, followed by related work in the RL setting. In both cases, we highlight the  insights that follows from the results in this paper.

\textbf{SGD for LSE}: As mentioned in the previous section, distributions underlying
SGD for LSE with bounded data 
is a subset of $\P_{\text{PSD},B}$ and hence is weakly admissible under a fixed constant step-size choice. 
However, we also noted that $\P_{\text{PSD},B}$ is not admissible. This seems to be at odds with the result of \citet{bach}
who prove that the MSE of SGD with CS-PR with an appropriate constant is bounded by
$\frac{C}{t}$ where $C>0$ only depends on $B$. The apparent contradiction is resolved by noting that 
{\em (i)} in SGD the natural loss is $\EE{\normsm{\thh_t-\ts}^2_{A_P}}$ with $A_P$ SPD, and 
{\em (ii)} the noise (arising due to the residual error) is ``structured'', i.e., its variance \todoc{What exactly is this?}
is bounded by $R\,A_P$ for some constant $R>0$ (see $\mathcal{A}3$, \cite{bach}).
\todoc{There is a subtlety here, because, if I believe you, 
\citet{bach} sells their result as a non-asymptotic result. How can they do this?
Can we be more explicit where in the paper $C_{P',\alpha}/t^2$ shows up -- see the notation below? We must be very explicit here.
The reviewers may not like this.
}

\textbf{Additive vs. multiplicative noise}: Analysis of LSA with CS-PR goes back to the work by \citet{polyak-judisky}, wherein they considered the additive noise setting i.e., $A_t=A$ for some deterministic Hurwitz matrix $A\in \R^{\dcd}$.
A key improvement in our paper is that we consider the `multiplicative' noise case, i.e., $A_t$ is non-constant random matrix. To tackle the multiplicative noise we use newer analysis introduced by \citet{bach}. However, since the general LSA setting (with Hurwitz assumption) does not enjoy special structures of the SGD setting of \citet{bach}, we make use of Jordan decomposition and similarity transformations in a critical way to prove our results, thus diverging from the line of analysis of 
\citet{bach}.

\textbf{Results for RL}: We are presented with data in the form of an \iid sequence $(\phi_t,\phi'_t,r_t)\in \R^d\times\R^d\times \R$. For a fixed constant $\gamma \in (0,1)$ define  $\Delta_t\eqdef \phi_t\phi_t^\top-\gamma \phi_t\phi_t^{'\top}$, $C_t=\phi_t\phi_t^\top$ and $b_t=\phi_r r_t$. In what follows, $\mu_t>0$ is an \emph{importance} sampling factor whose aim is to correct for mismatch in the (behavior) distribution with which the data was collected and the (target) distribution with respect to which one wants to learn. A factor $\mu_t=1,\,\forall t\geq 0$ will mean that no correction is required\footnote{This is known as the \emph{on-policy} case where the behavior is identical to the target. The general setting where $\mu_t>0$ is known as \emph{off-policy}.}. The various TD class of algorithms that can be cast as LSAs are given in \cref{tb:tdalgo}.
\begin{table}
\resizebox{\columnwidth}{!}{
\begin{tabular}{|c|l|p{60mm}|}\hline
Algorithm& Update & Remark\\ \hline
TD(0) 
	& $\begin{aligned}\theta_t&=\theta_{t-1}+\alpha_t (b_t -\Delta_t\theta_{t-1})\end{aligned}$ 
	& \cite{korda-prashanth}: $\alpha_t=O(\frac{1}{t})^\beta$, $\beta\in(0,1)$; PR-avg, ``on-policy''; $\EE{\norm{\eh_t}}= O(\frac{1}{\sqrt{t}}).$ \\ \hline
GTD/GTD2
	& $\begin{aligned} y_{t}&=y_{t-1}+\beta_t(\mu_t b_t -\mu_t \Delta_t \theta_{t-1}- Q_t y_{t-1})\\ \theta_t &=\theta_{t-1}+\alpha_t (\mu_t A^\top_t y_{t-1})\end{aligned}$
	& \cite{gtd2}: $\beta_t=\eta \alpha_t$, $\sum_{t\geq 0}\alpha_t=\infty$, $\sum_{t\geq 0}\alpha^2_t<\infty$; $e_t\ra 0$ as $t\ra\infty$ w.p.1. 
	\cite{gtdmp}: $\alpha_t=\beta_t=O(\frac{1}{\sqrt{t}})$; Projection+PR; $\,\norm{e_t} =O(t^{-\frac{1}{4}})$ w.h.p.\\\hline
\end{tabular}
}
\caption{Rates for TD algorithms available in the literature \cite{korda-prashanth,gtdmp,gtd2}. }
\label{tb:tdalgo}
\end{table}
The TD(0) algorithm is the most basic of the class of TD algorithms. 
An important shortcoming of TD(0) was its instability in the \emph{off-policy} case, which 
was successfully mitigated by the \emph{gradient temporal difference} learning GTD algorithm \cite{gtd2}. 
GTD was proposed by \citet{gtd}; its variants, namely GTD2 and TDC, were proposed later by \citet{gtd2}. 
The initial convergence analysis for GTD/GTD2/TDC was only asymptotic in nature \cite{gtd,gtd2} with diminishing step-sizes.

The most relevant to our results are those by \citet{korda-prashanth} in TD(0) and by \citet{gtdmp} in GTD. For the TD(0) case,  diminishing step-sizes $\alpha_t=O(\frac{1}{t})^\beta,\beta \in(0,1)$ with PR averaging is showed to exhibit a rate of $O(\frac{1}{t})$ decay for the MSE when $\beta\ra 1$ \cite{korda-prashanth}. 
In the case of GTD/GTD2 diminishing step-sizes $\alpha_t=O(\frac{1}{\sqrt{t}})$, projection of iterates and PR-averaging leads to a rate of $O(\frac{1}{\sqrt{t}})$ 
for the prediction error $\normsm{A_P\thh_t-b_P}^2$ with high probability \cite{gtdmp}. 
\citet{gtdmp} also suggest a new version of GTD based on stochastic mirror prox ideas, called the GTD-Mirror-Prox, 
which also shown to achieve an $O(\frac{1}{\sqrt{t}})$ rate for $\normsm{A_P\thh_t-b_P}^2$ with high probability under similar step-size choice that was used by them for the GTD.

All previous results on these RL algorithms assume that \Cref{assmp:lsa} holds (the Hurwitz assumption is satisfied by definition for on-policy TD(0), while it holds by design for the others). \todoc{Are we sure?
How about TDC? How about GTD-Mirror-Prox?}
Thus, \Cref{th:rate} applies to all of TD(0)/GTD/GTD2 with CS-PR in all cases considered in the literature.
In particular, our results show that the error in the GTD class of algorithms decay at the $O(\frac{1}{t})$ rate (even without use of projection or mirror maps) instead of $O(\frac{1}{\sqrt{t}})$, a major improvement on previously published results. In comparison to the TD(0) results by \citet{korda-prashanth}, \Cref{th:rate} is better in that it provides the bias/variance decomposition. 
While the $i.i.d$ assumption is made in much of prior work \cite{gtd2,gtdmp}, however, it is important to note that \citet{korda-prashanth} handle the Markov noise case which is not dealt with in this paper. \todoc{How about others? We should mention if others also assume iid.}
\section{Automatic Tuning of Step-Sizes}\label{sec:stepsizes}
It is straightforward to see from \eqref{eq:lsaintro} that $\alpha_t$ cannot be asymptotically increasing. We now present some heuristic arguments in favour of a constant step-size over asymptotically diminishing step-sizes in \eqref{eq:lsaintro}.
It has been observed that when the step-sizes of form $\alpha_t=\frac{c}{t}$ or $\alpha_t=\frac{c}{c+t}$ (for some $c>0$) are used, the MSE, $\EE{\normsm{\theta_t-\ts}^2}$, is not robust to the choice of $c>0$ \cite{korda-prashanth,bach-moulines}. In particular only a $O(\frac{1}{t^{{\mu c}/2}})$ decay can be achieved for the MSE, where $\mu$ is the smallest positive part of the eigenvalues of \todoc{I added: ``eigenvalues of''} $A_P$ \cite{bach-moulines}. Note that, in the case of LSA with CS-PR, \Cref{th:rate} guarantees a $O(\frac{1}{t})$ rate of decay for the MSE and the problem dependent quantities affect only the constants and not the exponent. Also, in the case of important TD algorithms such as GTD/GTD2/TDC, while the theoretical analysis uses diminishing step-sizes, the experimental results are with a constant step-size or with CS and PR averaging \cite{gtd2,gtdmp}. Independently, \citet{dann} also observe in their experiments that a constant step-size is better than diminishing step-sizes.\par
We would like to remind that in \Cref{sec:uniform} we showed that weak admissibility might not hold for all problem classes, and hence a uniform choice for the constant step-size might not be possible, However, motivated by \Cref{th:rate} and also by the usage of constant step-size in practice \cite{dann,gtd2,gtdmp}, we suggest a natural algorithm to tune the constant step-size, shown as \Cref{alg:tuning}.\par
\begin{algorithm}
\caption{Automatic Tuning of Constant Step-Size}
\begin{algorithmic}[1]
\STATE{Initialize: $\theta_0$, $\alpha=\alpha_{\max}$, $k$, $T$ }
\FOR{$t=1,2,\ldots, $}
\STATE{$\theta_t=\theta_{t-1}+\alpha (b_t-A_t\theta_{t-1}), \thh_t=\thh_{t-1}+\frac{1}{t+1}(\theta_t-\thh_{t-1})$}
\IF{$IsUnstable(\normsm{\thh_t},\ldots,\normsm{\thh_{(t-kT)\wedge 0}})=\mathrm{True}$}
\STATE{$\alpha=\alpha/2$}
\ENDIF
\ENDFOR
\end{algorithmic}
\label{alg:tuning}
\end{algorithm}
In \Cref{alg:tuning}, $T>0$ is a time epoch and $k$ is a given integer and $\alpha_{\max}>0$ is the maximum step-size that is allowable. From the Gronwall-Bellman lemma it follows that in \Cref{alg:tuning} $\norm{\theta_t}\leq C(1+e^{\beta t})$ with some $C>0$, where the sign of $\beta$ determines whether the iterates are bounded. 
Using this fact, we observe that the sequence $r_i=\frac{\normsm{\thh_{(t-kT+iT)\wedge 0}}}{\normsm{\thh_{(t-kT+(i-1)T)\wedge 0}}},i=1,\ldots,k$ should be ``roughly'' (making allowance for the persistent noise) decreasing and converge to $1$ when the step-size is large enough so that the iterates stay bounded and eventually converge. 
The idea is that the $IsUnstable()$ routine in \Cref{alg:tuning} calculates $\{r_i\}_i$ based on its input and returns true when any of these is larger than a preset constant $c>1$. By choosing a larger the constant $c$, the probability of false detection of a run-away event decreases rapidly, while still controlling for the probability of altogether missing a run-away event.\par
We ran numerical experiments on the class with $A_{P}=\begin{bmatrix} 1 &-10\\ 10 &1\end{bmatrix}$, $\sigma_b=0$ and $b_t=b,\forall t\geq 0$ (chosen such that $\ts=(1,1)^\top$) and $M_t,\,t \geq 0$ with varying $\sigma_A$'s. This problem class does not admit an apriori step-size (due to the unknown $\sigma_A$ and the dependence of step-size on $\sigma_A$) that prevents the explosion of MSE.
The results (see \Cref{fig:tune}) show that \Cref{alg:tuning} does find a problem dependent constant step-size  (within a factor of the best possible hand computed step-size) that avoids the MSE blow up. We chose $k=2$ and $T=5$, the preset constant was chosen to be $1.025$ and the results are for $\sigma_A=0,2,5,10,20$.
\Cref{alg:tuning} is oblivious of the data distribution, and the hand computed step-size is based on full problem information (i.e., $\sigma_A$).
 Further, the results (in the right plot of \Cref{fig:tune}) also confirm our expectation that higher step-sizes lead to faster convergence.

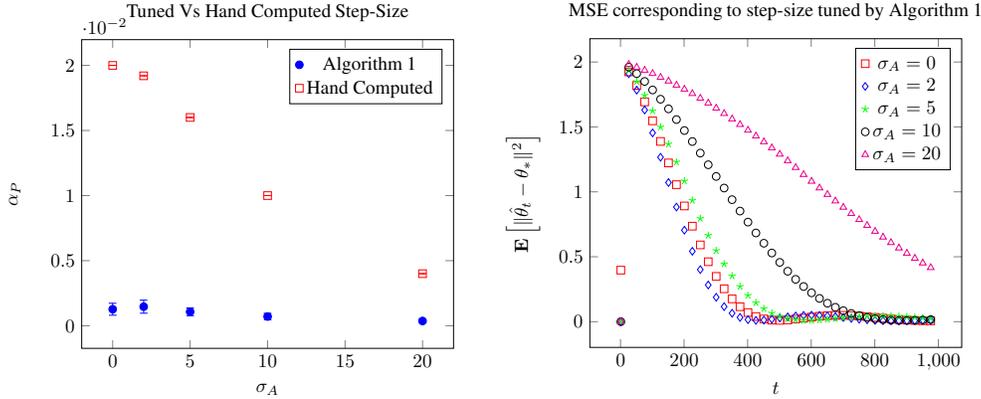
\begin{figure}
\resizebox{\columnwidth}{!}{
\begin{tabular}{cc}
\begin{tikzpicture}[scale=0.5]
\begin{axis}[
xlabel=$\sigma_A$,
ylabel=$\alpha_P$, legend pos= north east,
title= Tuned Vs Hand Computed Step-Size
]
\addplot[only marks, blue, mark=*, mark options={blue}, error bars/.cd,y dir=both,  y explicit] table [x index=0, y index=1, y error index=2]{./hur};
\addplot[only marks, red, mark=square, mark options={red}, error bars/.cd,y dir=both,  y explicit] table [x index=0, y index=1, y error index=2]{./hur_act};
\addlegendentry{\Cref{alg:tuning}}
\addlegendentry{Hand Computed}
\end{axis}
\end{tikzpicture}

&
\begin{tikzpicture}[scale=0.5]
\begin{axis}[
xlabel=$t$,
ylabel=$\EE{\normsm{\thh_t-\ts}^2}$, legend pos=north east,
title= MSE corresponding to step-size tuned by \Cref{alg:tuning}
]
\addplot[only marks,mark=square,red,each nth point=25] plot file {./hur1};
\addplot[only marks,mark=diamond,blue,each nth point=25] plot file {./hur2};
\addplot[only marks,mark=star,green,each nth point=25] plot file {./hur3};
\addplot[only marks,mark=o,black,each nth point=25] plot file {./hur4};
\addplot[only marks,mark=triangle,magenta,each nth point=25] plot file {./hur5};
\addlegendentry{$\sigma_A=0$}
\addlegendentry{$\sigma_A=2$}
\addlegendentry{$\sigma_A=5$}
\addlegendentry{$\sigma_A=10$}
\addlegendentry{$\sigma_A=20$}

\end{axis}
\end{tikzpicture}
\end{tabular}
}
\caption{The left plot shows comparison of the constant step-size $\alpha_P$ as function of $\sigma_A$ found by \Cref{alg:tuning} versus the constant step-size computed in closed form.  The right plot shows the performance of LSA with CS-PR (with the step-size choosen by \Cref{alg:tuning}) for various $\sigma_A$ values. The errors were insignificant and hence error bars are not shown in the right plot.}
\label{fig:tune}
\end{figure}

\section{Conclusion}
We presented a finite time performance analysis of LSAs with CS-PR and showed that the MSE decays at a rate $O(\frac{1}{t})$. Our results extended the analysis of \citet{bach} for SGD with CS-PR for the problem of linear least-squares estimation and \iid sampling to general LSAs with CS-PR. Due to the lack of special structures, our analysis for the case of general LSA cannot recover entirely the results of \citet{bach} who use the problem specific structures in their analysis.
Our results also improved the rates in the case of the GTD class of algorithms. We presented conditions under which a constant step-size can be chosen uniformly for a given class of data distributions. We showed a negative result in that not all data distributions `admit' such a constant step-size. This is a negative result from the perspective of TD algorithms in RL. We also argued that a problem dependent constant step-size can be obtained in an automatic manner and presented numerical experiments on a synthetic LSA. 

\bibliographystyle{plainnat}
\bibliography{ref}
\appendix
\section{Linear Algebra Preliminaries}\label{sec:appendix}
\subsection{Additional Notations}\label{sec:addnot}
For $x=a+ib\in \C$, we denote its real and imaginary parts by $\re{x}=a$ and $\im{x}=b$ respectively. Given a $x\in \C^d$, for $1\leq i \leq d$, $x(i)$ denotes the $i^{th}$ component of $x$.
For any $x\in \C$ we denote its modulus $\md{x}=\sqrt{\re{x}^2+\im{x}^2}$ and its complex conjugate by $\bar{x}=a-ib$.
We use $A\succeq 0$ to denote that the
square matrix $A$ is Hermitian and positive semidefinite (HPSD):
$A = A^*$, $\inf_x x^* A x\ge 0$. We use $A\succ 0$ to denote that the square matrix $A$ is Hermitian and positive definite (HPD): $A=A^*$, $\inf_x x^* A x > 0$.
For $A,B$ HPD matrices, $A\succeq B$ holds if $A-B\succeq 0$.
We also use $A\succ B$ similarly to denote that $A-B \succ 0$.
We also use $\preceq$ and $\prec$ analogously. We denote the smallest eigen value of a real symmetric positive definite matrix $A$ by $\lambda_{\min}(A)$.\par
We now present some useful results from linear algebra.

Let $B$ be a $\dcd$ block diagonal matrix given by $B=\begin{bmatrix} B_1 &0 &0 &\ldots &0 \\ 0 &B_2 &0 &\ldots &0  \\ \vdots &\vdots &\vdots &\vdots &\vdots \\ 0 &\ldots &0 &0 &B_k \end{bmatrix}$, where $B_i$ is a $d_i \times d_i$ matrix such that $d_i<d,\,\forall i=1,\ldots,k$ (w.l.o.g) and $\sum_{i=1}^k d_i=d$. We also denote $B$ as
\begin{align*}
B=B_1 \op B_2 \op \ldots B_k=\op_{i=1}^k B_i
\end{align*}

\subsection{Results in Matrix Decomposition and Transformation}
We will now recall Jordon decomposition.
\begin{lemma}\label{jordon}
Let $A\in \C^{\dcd}$ and $\{\lambda_i\in \C,i=1,\ldots,k\leq d \}$ denote its $k$ distinct eigenvalues.
There exists a complex matrix $V\in \C^{\dcd}$ such that $A=V\tL V^{-1}$, where
$\tL=\tL_1\op\ldots\op\tL_k$, where each $\tL_i,\,i=1,\ldots,k$ can further be written as $\tL_i= {\tL}^i_{1}\op \ldots \op {\tL}^i_{{l(i)}}$. Each of ${\tL}^i_{j},j=1,\ldots,l(i)$ is a $d^i_j\times d^i_j$ square matrix such that $\sum_{j=1}^{l(i)} d^i_j =d_i$ and has the special form given by
${\tL}^i_{j}=\begin{bmatrix} \lambda_i &1 &0 &\ldots &0 &0\\ 0 &\lambda_i &1 &0 &\ldots &0 \\ 0 &\vdots &\vdots &0 &\lambda_i &1 \\ 0 &\ldots &0 &0 &0 &\lambda_i \end{bmatrix}$.
\end{lemma}

\begin{lemma}\label{lm:simtran}
Let $A\in \C^{\dcd}$ be a Hurwitz matrix. There exists a matrix $U\in \gln$ such that $A=U\Lambda U^{-1}$ and $\Lambda^*+\Lambda$ is a real symmetric positive definite matrix.
\end{lemma}
\begin{proof}
It is trivial to see that for any $\Lambda\in \C^{\dcd}$, $\left(\Lambda^*+\Lambda\right)$ is Hermitian. We will use the decomposition of $A=V \tL V^{-1}$ in \Cref{jordon} and also carry over the notations in \Cref{jordon}. Consider the diagonal matrices $D^i_j=\begin{bmatrix} 1  &0 &0 &\ldots &0 &0\\ 0 &\re{\lambda_i} &0 &0 &\ldots &0 \\ 0 &\vdots &\vdots &0 &\re{\lambda_i}^{d^i_j-1} &0 \\ 0 &\ldots &0 &0 &0 &\re{\lambda_i}^{d^i_j} \end{bmatrix},\,\forall j=1,\ldots,l(i)$, $D^i=D^i_1 \op\ldots\op D^i_{l(i)},\,\forall i=1,\ldots,k$ and $D=D^1 \op\ldots\op D^k$.
It follows that $A=(VD) \Lambda (VD)^{-1}$, where $\Lambda$ is a matrix such that
$\Lambda=\Lambda_1 \op \ldots \op \Lambda_k$, where each $\Lambda_i,\,i=1,\ldots,k$ can further be written as
$A_i=\Lambda^i_{1} \op \ldots \op \Lambda^i_{{l(i)}}$. Each of $\Lambda^i_{j}$ is a $d^i_j\times d^i_j$ square matrix with the special form given by
$\Lambda^i_{j}=\begin{bmatrix} \lambda_i &\re{\lambda_i} &0 &\ldots &0 &0\\ 0 &\lambda_i &\re{\lambda_i} &0 &\ldots &0 \\ 0 &\vdots &\vdots &0 &\lambda_i &\re{\lambda_i} \\ 0 &\ldots &0 &0 &0 &\lambda_i \end{bmatrix}$.

Now we have $\frac{(\Lambda^*+\Lambda)}{2}=\op_{i=1}^k \op_{j=1}^{l(i)}\frac{\Lambda^{i*}_{j}+\Lambda^i_{j}}{2}$, where $\frac{\Lambda^{i*}_{j}+\Lambda^i_{j}}{2}=\begin{bmatrix} \re{\lambda_i} &\frac{\re{\lambda_i}}{2} &0 &\ldots &0 &0\\ \frac{\re{\lambda_i}}{2} &\re{\lambda_i} &\frac{\re{\lambda_i}}{2} &0 &\ldots &0 \\ 0 &\vdots &\vdots &0 &\re{\lambda_i} &\frac{\re{\lambda_i}}{2} \\ 0 &\ldots &0 &0 &\frac{\re{\lambda_i}}{2} &\re{\lambda_i} \end{bmatrix} $. Then for any $x=(x(i),i=1,\ldots,d)\in \C^d (\neq \mathbf{0})$, we have 
\begin{align*}
x^* \frac{(\Lambda^*+\Lambda)}{2} x &=\re{\lambda_i} \left(\sum_{i=1}^d \bar{x}{(i)} x(i)+\sum_{i=1}^{d-1} \frac{\bar{x}(i) x(i+1) + x(i)\bar{x}(i+1)}{2}\right) \\
&=\frac{\re{\lambda_i}}{2}\left(\md{x(1)}^2+ \md{x(d)}^2\right)+\frac{\re{\lambda_i}}{2}\left( \sum_{i=1}^{d-1} \md{x(i)}^2+\bar{x}(i) x(i+1) + x(i)\bar{x}(i+1)+\md{x(i+1)}^2 \right)\\
&>\frac{\re{\lambda_i}}{2}\left(\sum_{i=1}^d \md{x(i)+x(i+1)}^2 \right)\\
&> 0
\end{align*}
\end{proof}
\section{Proofs}\label{sec:proofs}
\subsection{LSA with CS-PR for Positive Definite Distributions}
In this subsection, we re-write \eqref{eq:lsa} and \Cref{assmp:lsa} to accomodate complex number computations and in addition assume that $P$ is \emph{positive definite}. To this end,
\begin{subequations}\label{eq:lsacmplx}
\begin{align}
\label{conststep}&\text{LSA:} &\theta_t&=\theta_{t-1}+\alpha(b_t-A_t\theta_{t-1}),\\
\label{iteravg}&\text{PR-Average:} &\thh_t&=\frac{1}{t+1}{\sum}_{i=0}^{t}\theta_i,
\end{align}
\end{subequations}
where $\thh_t, \theta_t \in \C^{d}$. We now assume,
\begin{assumption}\label{assmp:lsacmplx}
\begin{enumerate}[leftmargin=*, before = \leavevmode\vspace{-\baselineskip}]
\item \label{dist} $(b_t, A_t)\sim (P^b,P^A), t\geq 0$ is an \iid sequence, where $P^b$ is a distribution over $\C^d$ and $P^A$ is a distribution over $\C^{\dcd}$. We assume that $P$ is positive definite.
\item \label{matvar} The martingale difference sequences\footnote{$\EE{M_t|\F_{t-1}}=0$ and $\EE{N_t|\F_{t-1}}=0$} $M_t\eqdef A_t-A_{P}$ and $N_t\eqdef b_t-b_{P}$ associated with $A_t$ and $b_t$ satisfy the following
\begin{align*}\E\left[ \norm{M_t}^2\mid\F_{t-1}\right]\leq \sigma^2_{A_P}, \, \E[N_t^* N_t]=\sigma^2_{b_P}.\end{align*}
\item $A_P$ is invertible and there exists a $\ts=A^{-1}_Pb_P$.
\end{enumerate}
\end{assumption}

We now define the error variables and present the recurison for the error dynamics. In what follows, definitions in \Cref{sec:def} and \Cref{sec:prob} continue to hold.
\begin{definition}\label{def:err}
\begin{itemize}[leftmargin=*, before = \leavevmode\vspace{-\baselineskip}]
\item Define error variables $e_t\eqdef \theta_t-\ts$ and $\eh_t\eqdef \thh_t-\ts$.
\item Define $\forall\, t\geq 0$ random vectors $\zeta_t\eqdef b_t-b-(A_t-A_P)\ts$.
\item Define constants $\sigma_1^2\eqdef\sigma_A^2\norm{\ts}^2+\sigma_b^2$ and $\sigma_2^2\eqdef\sigma_A^2\norm{\ts}$. Note that $\EE{\norm{\zeta_t}^2}\leq \sigma_1^2$ and $\EE{\norm{M_t\zeta_t}}\leq \sigma_2^2$.
\item Define $\forall\,i\geq j$, the random matrices $F_{i,j}=(I-\alpha A_i)\ldots (I-\alpha A_j)$ and $\forall,\,i<j$ $F_{i,j}=\I$.
\end{itemize}
\end{definition}

\paragraph{Error Recursion} Let us now look at the dynamics of the error terms defined by
\begin{align}\label{eq:errrec}
\begin{split}
\theta_t&=\theta_{t-1}+\alpha\big(b_t-A_t\theta_{t-1}\big)\\
\theta_t-\ts&=\theta_{t-1}-\ts+\alpha\big(b_t-A_t(\theta_{t-1}-\ts+\ts)\big)\\
e_t&=(I-\alpha A_t)e_{t-1}+\alpha(b_t -b -(A_t-A)\ts)\\
e_t&=(I-\alpha A_t)e_{t-1}+\alpha\zeta_t
\end{split}
\end{align}

\begin{lemma}\label{lm:pd}
Let $P$ be a distribution over $\C^d\times \C^{\dcd}$ satisfying \Cref{assmp:lsacmplx}, then there exists an $\alpha_P>0$ such that $\rhod{P}>0$ and $\rhos{P}>0,~\forall \alpha \in (0,\alpha_P)$.
\end{lemma}
\begin{proof}
\begin{align*}
\rhos{P}&\stackrel{(a)}{=}\inf_{x:\norm{x}=1}x^* (A_P^*+A_P)x -\alpha x^*\EE{A_t^* A_t} x\\
&\stackrel{(b)}{=}\inf_{x:\norm{x}=1}x^* (A_P^*+A_P)x -\alpha x^* A^*_P A_P -\alpha x^* \EE{M_t^* M_t} x\\
&\stackrel{(c)}{\geq} \lambda_{\min}(A^*_P+A_P)-\alpha \norm{A_P}^2-\sigma^2_A
\end{align*}
The proof is complete by choosing $\alpha_P<\frac{\lambda_{\min}(A^*_P+A_P)}{\norm{A_P}^2+\sigma^2_A}$. Here $(a)$ follows from definition of $\rhos{P}$ in \Cref{def:dist}, $(b)$ follows from the fact that $M_t$ is a martingale difference term (see \Cref{assmp:lsacmplx}) and $(c)$ follows from the fact that for a real symmetric matrix $M$ the smallest eigen value is given by $\lambda_{\min}=\inf_{x:\norm{x}=1} x^* M x $.
\end{proof}

\begin{lemma}[Product unroll lemma]\label{lem:genunroll}
Let $t>i\ge 1$, $x,y\in \C^d$ be $\F_{i}$-measurable random vectors. Then,
\begin{align*}
\E[x^* F_{t,i+1}y|\F_i]=x^* (I-\alpha A_P)^{t-i} y\,.
\end{align*}
\end{lemma}
\begin{proof}
By the definition of $F_{t,i+1}$,
and because $F_{t-1,i+1} = (I-\alpha A_{t-1}) \dots (I-\alpha A_{i+1})$ is $\F_{t-1}$-measurable,
as are $x$ and $y$,
\begin{align*}
\EE{x^* F_{t,i+1} y | \F_{t-1} } &= x^\top \EE{ (I-\alpha A_t) | \F_{t-1} } F_{t-1,i+1} y\\
&=x^*  (I-\alpha A_P)  F_{t-1,i+1} y\,.
\end{align*}
By the tower-rule for conditional expectations and our measurability assumptions,
\begin{align*}
\EE{x^* F_{t,i+1} y | \F_{t-2} }
&=x^*  (I-\alpha A_P)  \EE{F_{t-1,i+1} |\F_{t-2}} y\\
&= x^* (I-\alpha A_P)^2 F_{t-2,i+1} y\,.
\end{align*}
Continuing this way we get
\begin{align*}
\EE{x^* F_{t,i+1} y | \F_{t-j} }
= x^* (I-\alpha A_P)^j F_{t-j,i+1} y\,, \quad j=1,2,\dots,t-i\,.
\end{align*}
Specifically, for $j=t-i$ we get
\begin{align*}
\EE{x^* F_{t,i+1} y | \F_{i} }  = x^* (I-\alpha A_P)^{t-i} y\,.
\end{align*}
\end{proof}

\begin{lemma}\label{noisecancel}
Let $t>i\ge 1$ and let $x\in \C^d$ be a $\F_{i-1}$-measurable random vector. Then,
$\E[x^* F_{t,i+1}\zeta_{i}]=0$.
\end{lemma}
\begin{proof}
By \Cref{lem:genunroll},
\begin{align*}
\EE{x^* F_{t,i+1} \zeta_i | \F_{i} }  = x^* (I-\alpha A_P)^{t-i} \zeta_i\,.
\end{align*}
Using the tower rule,
\begin{align*}
\EE{x^* F_{t,i+1} \zeta_i | \F_{i-1} }
= x^* (I-\alpha A_P)^{t-i}\EE{ \zeta_i | \F_{i-1} }= 0\,.
\end{align*}
\end{proof}

\begin{lemma}\label{lem:unroll}
For all $t>i\ge 0$, $\E \ip{e_i,F_{t,i+1} e_i}=\E\ip{e_i,(I-\alpha A_P)^{t-i} e_i}$.
\end{lemma}
\begin{proof}
The lemma follows directly from \Cref{lem:genunroll}. Indeed,
$\theta_i$ depends only on $A_1,\dots,A_{i},b_1,\dots,b_{i}$, $\theta_i$ and so is $e_i$ $\F_i$-measurable.
Hence, the lemma is applicable and implies that
\begin{align*}
\EE{ \ip{e_i, F_{t,i+1} e_i} | \F_i } =
\EE{ \ip{e_i, (I-\alpha A_P)^{t-i} e_i} | \F_i }\,.
\end{align*}
Taking expectation of both sides gives the desired result.
\end{proof}

\begin{lemma}\label{innerproduct}
Let $i>j \ge 0$ and let $x\in \R^d$ be an $\F_j$-measurable random vector.
Then,
\begin{align*}
\E\ip{F_{i,j+1}x,F_{i,j+1}x}\leq (1-\alpha \rhos{P})^{i-j}\E\norm{x}^2\,.
\end{align*}
\end{lemma}
\begin{proof}
Note that
$S_t\doteq \EE{ (I-\alpha A_t)^* (I-\alpha A_t) | \F_{t-1} }
= I - \alpha (A_P^* + A_P) + \alpha^2 \EE{ A_t^* A_t | \F_{t-1} }$.
Since $(b_t,A_t)_t$ is an independent sequence, $\EE{ A_t^* A_t|\F_{t-1}} = \EE{ A_1^* A_1 }$.
Now, using the definition of $\rhos{P}$ from \Cref{def:dist}
$\sup_{x:\norm{x}= 1} x^\top S_t x = 1 - \alpha \inf_{x:\norm{x}=1} x^\top (A_P^* + A_P - \alpha \EE{A_1^\top A_1}) x
= 1-\alpha \rhos{P}$.
Hence,
\begin{align*}
&\EE{\ip{F_{i,j+1}x,F_{i,j+1}x}|\F_{i-1} }\\
&= \EE{x^* F_{i-1,j+1}^\top (I-\alpha A_i)^* (I-\alpha A_i) F_{i-1,j+1} x\,|\,\F_{i-1}}\\
&=(x F_{i-1,j+1})^* \, S_i \, F_{i-1,j+1} x\\
&\le (1-\alpha \rhos{P}) \, \ip{ F_{i-1,j+1} x, F_{i-1,j+1} x} \\
& \le (1-\alpha \rhos{P})^2\, \ip{ F_{i-2,j+1} x, F_{i-2,j+1} x} \\
& \quad \vdots \\
& \le (1-\alpha \rhos{P})^{i-j}\, \norm{x}^2\,.
\end{align*}
\end{proof}

\begin{theorem}\label{th:pdrate}
Let $\eh_t$ be as in \Cref{def:err}. 
Then
\begin{align}
\E[\norm{\eh_t}^2]
\leq \left(1+\frac2{\alpha\rhod{P}}\right)\, \frac1{\alpha\rhos{P}}\, \,
\left(\frac{\norm{e_0}^2}{(t+1)^2}+ \frac{\alpha^2(\sigma_1^2)+\alpha \sigma_2^2\norm{e_0}}{t+1} \right)\,.
\end{align}
\end{theorem}
\begin{proof}

\begin{align*}
e_t
& = (I-\alpha A_t) (I-\alpha A_{t-1}) e_{t-2}\\ &+ \alpha (I-\alpha A_t) \zeta_{t-1} +\alpha \zeta_t \\
& \quad \vdots\\
& = (I-\alpha A_t) \cdots (I-\alpha A_1) e_0\\ &+ \alpha (I-\alpha A_t) \cdots (I-\alpha A_2) \zeta_1 \\
& + \alpha (I-\alpha A_t) \cdots (I-\alpha A_3) \zeta_2\\
&  \quad \vdots \\
&+ \alpha \zeta_t\,,
\end{align*}
which can be written compactly as
\begin{align}
\label{eq:etft}
e_t = F_{t,1} e_0 + \alpha (F_{t,2} \zeta_1 + \dots + F_{t,t+1} \zeta_t )\,,
\end{align}
\begin{align*}
\eh_t=\frac{1}{t+1}{\sum}_{i=0}^{t}e_i
=\frac{1}{t+1}&\Big\{{\sum}_{i=0}^{t} F_{i,1} e_0 \\
&+ \alpha \sum_{i=1}^{t} \left(\sum_{k=i}^{t} F_{k,i+1} \right)\zeta_i \Big\} ,
\end{align*}
where in the second sum we flipped the order of sums and swapped the names of the variables that the sum runs over.
It follows that \todoc{We should rather use $C$ instead of $H$ here?}
\begin{align*}
\E[\norm{\eh_t}^2]&=\E\ip{\eh_t,\eh_t}
=\frac{1}{(t+1)^2} \sum_{i,j=0}^t \E\ip{e_i,e_j}\,.
\end{align*}
Hence, we see that it suffices to bound $\EE{\ip{ e_i,  e_j }}$.
There are two cases depending on whether $i=j$. When $i< j$,
\begin{align*}
\E\ip{e_i,e_j}
&=\E \ip{e_i,\big[F_{j,i+1} e_i+\alpha\textstyle\sum_{k=i+1}^j F_{j,k+1}\zeta_{k}\big]}\\
&=\E\ip{e_i,F_{j,i+1} e_i}  \text{(from \Cref{noisecancel})}\\
&=\E\ip{e_i, (I-\alpha A)^{j-i} e_i} \text{(from \Cref{lem:unroll})}
\end{align*}
and therefore
\begin{align*}
\label{inter}
\sum_{i=0}^{t-1}\sum_{j=i+1}^t \E\ip{e_i,e_j}
&=\frac1{\alpha\rhod{P}} {\sum}_{i=0}^{t-1}\E\ip{e_i,e_i}\\
&\leq \frac2{\alpha\rhod{P}}{\sum}_{i=0}^{t}\E\ip{e_i,e_i}\,.
\end{align*}
Since $\sum_{i,j}\cdot{} = \sum_{i=j}\cdot{} + 2 \sum_i \sum_{j>i} \cdot{}$,
\begin{align*}
{\sum}_{i=0}^{t}{\sum}_{j=0}^{t} \E\ip{e_i,e_j}&= \left(1+\frac2{\alpha\rhod{P}}\right){\sum}_{i=0}^{t}\E\ip{e_i,e_i}\,.
\end{align*}
Expanding $e_i$ using \eqref{eq:etft} and then using \Cref{innerproduct} and \Cref{assmp:lsacmplx}
\begin{align*}
\E\ip{e_i,e_i}&=\E\ip{F_{i,1}e_0,F_{i,1}e_0}+\alpha^2{\sum}_{j=1}^i\E\ip{ F_{i,j+1}\zeta_j, F_{i,j+1}\zeta_j}+\alpha\sum_{j=1}^i  \E\ip{F_{i,1} e_0, F_{i,j+1}\zeta_j}\\
&\leq (1-\alpha\rhos{P})^i\norm{e_0}^2+ \alpha^2\frac{{\sigma}_1^2}{\alpha \rhos{P}}+ \alpha \frac{{\sigma}^2_2 \norm{e_0}}{\alpha\rhos{P}}\,,
\end{align*}
and so
\begin{align*}
{\sum}_{i=0}^{t}{\sum}_{j=0}^{t} \E\ip{e_i,e_j}
&\leq \left(1+\frac2{\alpha\rhod{P} }\right)\, \frac1{\alpha\rhos{P}}\, (t(\alpha^2{\sigma}_1^2+\alpha {\sigma}^2_2\norm{e_0}) +\norm{e_0}^2)\,.
\end{align*}
Putting things together,
\begin{align}
\E[\norm{\eh_t}^2]
\leq \left(1+\frac2{\alpha\rhod{P}}\right)\, \frac1{\alpha\rhos{P}}\, \,
\left(\frac{\norm{e_0}^2}{(t+1)^2}+ \frac{\alpha^2(\sigma_1^2)+\alpha \sigma_2^2\norm{e_0}}{t+1} \right)\,.
\end{align}
\end{proof}

\paragraph{Proof of \Cref{lm:hur}}
\begin{lemma} 
Let $P$ be a distribution over $\R^d\times \R^{\dcd}$ satisfying \Cref{assmp:lsa}, then there exists an $\alpha_{P_U}>0$ and $U\in \gln$ such that $\rhod{P_U}>0$ and $\rhos{P_U}>0,~\forall \alpha \in (0,\alpha_P)$.
\end{lemma}
\begin{proof}
We know that $A_P$ is Hurwitz and from  \Cref{lm:simtran} it follows that there exists an $U\in \gld$ such that  $\Lambda=U^{-1} A_P U$ and $(\Lambda^*+\Lambda)$ is real symmetric and positive definite. Using \Cref{def:simdist}, we have $A_{P_U}=\Lambda$ and from \Cref{lm:pd} we know that there exists an $\alpha_{P_U}$ such that $\rhod{P_U}>0$ and $\rhos{P_U}>0,~\forall \alpha \in (0,\alpha_{P_U})$.

\end{proof}

\begin{lemma}[Change of Basis]\label{lm:cb}
Let $P$ be a distribution over $\R^d\times \R^{\dcd}$ as in \Cref{assmp:lsa} and let $U$ be chosen according to \Cref{lm:hur}. Define $\gamma_t\eqdef U^{-1}\theta_t,\,,\gamma_*\eqdef U^{-1}\ts$, then
\begin{align}
\EE{\norm{\gamma_t-\gamma_*}^2}
\leq
\left(1+\frac2{\alpha\rhod{P_U}}\right)\frac{\norm{U^{-1}}^2}{\alpha \rhos{P_U}}
\left(\frac{\norm{\theta_0-\ts}^2}{(t+1)^2}+ \frac{\alpha^2(\sigma_P^2\norm{\ts}^2+\sigma_b^2)+\alpha (\sigma_P^2\norm{\ts})\norm{\theta_0-\ts}}{t+1} \right)\,.
\end{align}
where $\gh_t=\frac{1}{t+1}\sum_{s=0}^t \gamma_s$.
\end{lemma}
\begin{proof}
Consider the modified error recursion in terms of $z_t\eqdef \gamma_t-\gamma_*$
\begin{align}\label{eq:newerrrec}
\begin{split}
e_t&=(I-\alpha A_t)e_{t-1}+\alpha\zeta_t\\
U^{-1}e_t&=(I-\alpha U^{-1}A_t U) U^{-1}e_{t-1}+ \alpha U^{-1}\zeta_t\\
z_t&=(I-\alpha \Lambda_t) z_{t-1}+\alpha H_t,
\end{split}
\end{align}
where  $\Lambda_t=U^{-1}A_t U$ and $H_t=U^{-1}\zeta_t$. Note that the error recursion in $z_t$ might involve complex computations (depending on whether $U$ has complex entries or not), and hence \eqref{eq:lsacmplx} and \Cref{assmp:lsacmplx} are useful in analyzing $z_t$.
We know that $\EE{\norm{H_t}^2}\leq \norm{U^{-1}}^2\EE{\norm{\zeta_t}}$ and $\EE{\norm{\Lambda_t H_t}}=\EE{\norm{U^{-1}A_t UU^{-1}\zeta_t}}=\EE{\norm{U^{-1}A_t \zeta_t}}\leq \norm{U^{-1}}\EE{\norm{A_t\zeta_t}}=\norm{U^{-1}}\sigma_2^2$. Now applying \Cref{th:pdrate} to $\hat{z}_t\eqdef \frac{1}{t+1}\sum_{s=0}^t z_t$, we have
\begin{align}
\E[\norm{\zh_t}^2]
&\leq \left(1+\frac2{\alpha\rhod{P_U}}\right)\, \frac1{\alpha\rhos{P_U}}\, \,
\left(\frac{\norm{z_0}^2}{(t+1)^2}+ \frac{\alpha^2(\norm{U^{-1}}^2\sigma_1^2)+\alpha (\norm{U^{-1}}\sigma_2^2)\norm{z_0}}{t+1} \right)\,\\
&\leq \left(1+\frac2{\alpha\rhod{P_U}}\right)\, \frac1{\alpha\rhos{P_U}}\, \,
\left(\frac{\norm{U^{-1}}^2\norm{e_0}^2}{(t+1)^2}+ \frac{\alpha^2(\norm{U^{-1}}^2\sigma_1^2)+\alpha (\norm{U^{-1}}\sigma_2^2)\norm{U^{-1}}\norm{e_0}}{t+1} \right)\,
\end{align}

\end{proof}

\paragraph{Proof of \Cref{th:rate}}
Follows by substituting $\theta_t=U\gamma_t$ in \Cref{lm:cb}.
\paragraph{Proof of \Cref{th:lb}}

Consider the LSA with $(b_t,A_t)\sim P$ such that $b_t=(N_t,0)^\top\in\R^2$ is a zero mean \iid random variable with variance $\sigma^2_b$, and $A_t=A,\,\forall t\geq 0$, where $A=A_P=\begin{bmatrix} \lambda_{\min} &0\\ 0& \lambda_{\max}\end{bmatrix}$, for some $\lambda_{\max}>\lambda_{\min}>0$. Note that in this example $\ts=0$.
By choosing $\alpha<\frac2{\lambda_{\max}}$, in this case it is straightforward to write the expression for $\eh_t$ explicitly as below:
\begin{align*}
\eh_t&=\frac{1}{t+1}\sum_{s=0}^t e_t = \frac{1}{t+1}\sum_{s=0}^t (I-\alpha A_P)^{t-s} e_0 + \sum_{s=1}^t \sum_{i=s}^t (I-\alpha A_P)^{i-s} b_s\\
&=\frac{1}{t+1}(\alpha A_P)^{-1}\left[\left(I-(I-\alpha A_P)^{t+1}\right)e_0 + \sum_{s=1}^t \left(I-(I-\alpha A_P)^{t+1-s}\right) b_s\right]\,.
\end{align*}
Thus,
\begin{align*}
\EE{\norm{\eh_t}^2}&\stackrel{(a)}{=}\frac{1}{(t+1)^2}\Big[\norm{(\alpha A_P)^{-1}\left(I-(I-\alpha A_P)^{t+1}\right)e_0}^2 \\ 
&+\sum_{s=1}^t \norm{(\alpha A_P)^{-1}\left(I-(I-\alpha A_P)^{t+1-s}\right)b_s}^2\Big]\,,
\end{align*}
and hence
\begin{align*}
\EE{\norm{\eh_t}^2}
& \geq \EE{\eh^2_t(1)}\stackrel{(b)}{=}\frac{1}{(t+1)^2}(\alpha \lambda_{\min})^{-2}\Big[\left(1-(1-\alpha \lambda_{\min})^{t+1}\right)^2 \theta^2_0(1)\\
& + \frac{1}{(t+1)^2}\sum_{s=1}^t\left(1-(1-\alpha \lambda_{\min})^{t+1-s}\right)^2 b^2_s(1) \Big]\,.
\end{align*}
Here $(a)$ and $(b)$ follows from the \iid assumption. Note that in this example, $\rhos{P}=\rhod{P}=2\lambda_{\min} -\alpha \lambda_{\min}^2=\lambda_{\min}(2-\alpha \lambda_{\min})$, and $\norm{\ts}=0$ and $\sigma^2_A=0$. Further, the result follows by noting the fact that noting the fact that $\norm{b_t}^2=b_t(1)^2$ and $\norm{\theta_t}^2=\theta_t(1)^2$.

\paragraph{Proof of \Cref{lm:notwad}}
Fix an arbitrary $\alpha>0$. We show that there exists $P\in \P$ such that $\rho_\alpha(P)<0$.
For $\epsilon \in (0,1/2)$ let $P=(P^V,P^M)$ be the distribution such that $P^M$ is supported on $\{-I,I\}$ and takes on the value of $I$ with probability $1/2+\epsilon$. Then $A_P = 2\epsilon I \succ 0$, hence $P\in \P_1$. Further, $Q_P = I$.
Hence, $\rhos{P} = 4\epsilon-\alpha$. Hence, if $\epsilon<\alpha/4$, $\rho_\alpha(P)<0$.

\paragraph{Proof of \Cref{lm:ppsdbwd}}
Since $\P_{\text{PSD},B}$ is supported on the set of positive semi-definite matrices, we know for any $A\in \R^{\dcd}$ that is PSD, we can consider the SVD of $A$: $A = U \Lambda U^\top$ where $U$ is orthonormal and $\Lambda$ is diagonal with
nonnegative elements. Note that $\Lambda \preceq B\, \I$ and thus $\Lambda^2 \preceq B \Lambda$.
Then for any $x\in \R^d$, $x^\top A^\top A x = x^\top U \Lambda^2 U^\top x \le B x^\top U \Lambda U^\top x = B x^\top A x$.
Taking expectations we find that $x^\top C_P x \le B x^\top A_P x$.
Hence, $\rhos{P_{\text{PSD},B}} = 2 x^\top A_P x - \alpha x^\top C_P x \ge (2- \alpha B ) \,x^\top A_P x $.
Thus, for any $\alpha<2/B$, $\rho_\alpha(P)>0$.

\paragraph{Proof of \Cref{lm:ppsdbna}}
Consider the case when the smallest eigenvalue of $A_P$ is $0$.

\end{document}